\documentclass[letterpaper]{article} % DO NOT CHANGE THIS
\usepackage{arxiv}  % DO NOT CHANGE THIS
\usepackage[hyphens]{url}  % DO NOT CHANGE THIS
\usepackage{graphicx} % DO NOT CHANGE THIS
\usepackage{natbib}  % DO NOT CHANGE THIS AND DO NOT ADD ANY OPTIONS TO IT
\usepackage{caption} % DO NOT CHANGE THIS AND DO NOT ADD ANY OPTIONS TO IT
\usepackage{algorithm}
\usepackage{algorithmic}
\usepackage{amsmath}
\usepackage{amsfonts}
\usepackage{amsthm}

\newtheorem{lemma}{Lemma}[section]
\newtheorem{proposition}[lemma]{Proposition}
\newtheorem{corollary}[lemma]{Corollary}

\usepackage{listings}
\DeclareCaptionStyle{ruled}{labelfont=normalfont,labelsep=colon,strut=off} % DO NOT CHANGE THIS
\floatstyle{ruled}
\newfloat{listing}{tb}{lst}{}
\floatname{listing}{Listing}

\usepackage{booktabs}

\newcommand{\Method}{GCPO }
\usepackage[table]{xcolor}
\usepackage{multirow}
\usepackage{amsmath}
\usepackage{newfloat}
\usepackage{amsfonts}
\nocopyright
\title{GCPO: Diagnosing and Constraining Subspace Geometry in Rollout RL for LLMs}
\author{
    Kai Yang\textsuperscript{\rm 1,2,3}\equalcontrib, 
    Jingwei Xu\textsuperscript{\rm 3}\equalcontrib, 
    Wanyu Wang\textsuperscript{\rm 4}, 
    Kai-Yuan Guo\textsuperscript{\rm 5}, 
    Zhenbo Yu\textsuperscript{\rm 3,6}\corresponding, 
    Yi Wang\textsuperscript{\rm 2}\corresponding\thanks{Contributions made while affiliated with Shanghai AI Lab.}, 
    Yu Qiao\textsuperscript{\rm 1,2}\corresponding
}
\affiliations{
    \textsuperscript{\rm 1}School of Computer Science, Shanghai Jiao Tong University \\
    \textsuperscript{\rm 2}Shanghai Artificial Intelligence Laboratory \\
    \textsuperscript{\rm 3}NovaCore \\
    
    \textsuperscript{\rm 4}SJTU Paris Elite Institute of Technology, Shanghai Jiao Tong University \\
    \textsuperscript{\rm 5}School of Artificial Intelligence and Automation, Huazhong University of Science and Technology \\
    \textsuperscript{\rm 6}School of Artificial Intelligence, Shanghai Jiao Tong University \\

    {\{icarus1411, wwyspeit\}@sjtu.edu.cn, \{jingwei.xu, zhenbo.yu\}@deepnovacore.com, gky@hust.edu.cn, wygamle@gmail.com, qiaoyu@pjlab.org.cn}

}

\begin{document}

\maketitle

%comment：
%1.Principal-Subspace和Intrusion在这里缺少严格定义，应该不能直接讲，reviewer大概率get不到精确含义
%2.第一句“Rollout-based......”不用讲太多，甚至不用提，建议尽快进入讲问题的第二句。
%3.我直觉上需要提：principal subspace 捕捉的是预训练或先前任务的核心知识，那么侵入它确实可能导致遗忘，而不是直接说现在的RL会遗忘
%4.要highlight咱们的performance和抗遗忘能力（maybe还有entropy相关的）？
%5.可以斟酌一下要不要体现simple yet effectivede的感觉
%这是一个示例，供参考：（还缺少实验结果的数字支撑）
\begin{abstract}
On-policy rollout methods such as GRPO are central to post-training of large language models. Yet, they frequently suffer from training instabilities, cross-task capability degradation, and response-length inflation. Although prior work has characterized the subspace geometry of aggregate updates, the stepwise variation of this geometry and its relationship to model performance remain unclear. We introduce \emph{Principal-Subspace Overlap}, a dimension-corrected measure of individual rollout updates relative to the dominant singular subspaces of pretrained weights. Despite low average overlap, transient spikes often precede performance degradation. To address this, we propose \textbf{GCPO} (\textbf{G}eometrically \textbf{C}onstrained \textbf{P}olicy \textbf{O}ptimization), which applies hard bilateral orthogonal projections to constrain updates to the complementary subspaces, preventing such excursions by construction. Across mathematical reasoning, code generation, and tool-use tasks on Qwen3-8B and GLM4-9B, GCPO consistently outperforms GRPO and recent variants, including DAPO and GSPO, improving over the base models and the strongest baseline by up to 27.69 and 2.37 points, respectively. Furthermore, GCPO preserves general capabilities, eliminates response-length inflation, and stabilizes policy entropy. Our findings provide a new diagnostic lens and a principled design perspective for stable reinforcement learning post-training.

\end{abstract}

\begin{links}
    \link{Code \& Datasets}{https://github.com/Icarus1411/GCPO}
    % \link{Datasets}{https://aaai.org/example/datasets}
    % \link{Extended version}{https://github.com/Icarus1411/GCPO}
\end{links}

\section{Introduction}

Rollout-based reinforcement learning (RL) has become a widely used post-training mechanism for improving large language models (LLMs) on mathematical reasoning, coding, and tool use \cite{PPO, RL-survey, gao2023scaling}. Its optimization data, however, are generated by the evolving policy itself: each update changes both the policy and the rollout distribution used to construct the next update. This feedback loop can yield unstable training, degradation on capabilities outside the optimized task, and response-length inflation \cite{harmon2025mapping, ouyang2022training}.

Most existing remedies stabilize RL through objective- or policy-level controls. KL regularization constrains divergence from a reference policy, clipping limits large likelihood-ratio changes, and reward design modifies the optimization signal \cite{GSPO, DAPO, GMPO}. While effective, these methods may treat different updates as similar if they have similar policy-level statistics, even when their directions in parameter space differ. Because updates of similar magnitude may interact differently with the structured parameter space of the pretrained model, we ask whether their stepwise directions provide a complementary diagnostic of training dynamics.

We focus on the dominant singular subspaces of each pretrained linear weight matrix. The top right and left singular vectors identify the dominant input
and output directions, respectively. Together, they define the optimal rank-$k$ approximation of the pretrained transformation~\cite{EYM}. Updates
overlapping these subspaces can therefore alter structurally prominent components of the pretrained transformation. Rather than interpreting these directions as a literal decomposition of semantic knowledge, we treat them as a functionally distinguished yet tractable structural reference. Prior analyses further show that RL updates are predominantly off-principal when averaged over training \cite{GeometryOPD2026, DSSU, schneider2024identifyingpolicygradientsubspaces}. This aggregate regularity leaves open an important question: do individual updates exhibit transient principal-subspace overlap, and how is such overlap related to training performance?

Our stepwise analysis reveals transient entry of individual updates into the pretrained principal subspaces. We decompose each realized update into four blocks according to its left- and right-side overlap with these subspaces, and subtract the overlap expected under an isotropic null. Although most update energy lies in the doubly orthogonal block, aggregate statistics obscure intermittent spikes in \emph{excess principal-subspace overlap}.
Across the runs we examine, sustained or repeated spikes coincide
with, and often precede, declines in validation accuracy. We
therefore treat overlap as a geometric correlate of unstable phases, motivating updates that preserve the selected pretrained mappings by remaining in their bilateral orthogonal complement.

We instantiate this hypothesis in \textbf{GCPO} (\textbf{G}eometrically \textbf{C}onstrained \textbf{P}olicy \textbf{O}ptimization), which augments rollout-based policy optimization with a hard geometric constraint on the effective policy update. At each rollout iteration, GCPO optimizes the underlying policy objective while requiring every adapted layer update to lie in the bilateral orthogonal complement of the selected pretrained principal subspaces. GCPO thus changes the admissible directions of policy improvement, rather than introducing another policy-level penalty. The constraint is complementary to KL regularization: KL controls policy change in output space, whereas GCPO controls the feasible directions of parameter change. It also has a precise layer-level guarantee: for inputs in a selected principal input subspace, the adapted layer's response is unchanged. The remaining complement is large, retaining a high-dimensional feasible complement for task adaptation.

Comprehensive evaluations on Qwen3-8B and GLM4-9B cover mathematical reasoning, code generation, and tool use. Across all tasks, GCPO achieves the best accuracy, outperforming the strongest baseline by 1.02--2.37 points (±std over 3 seeds) and the corresponding instruction-tuned model by 7.09--27.69 points. When trained on mathematics and evaluated on other tasks, it attains the strongest worst-case retention. It also shows smoother accuracy and policy-entropy trajectories and suppresses response-length inflation.

In summary, our main contributions are threefold:
\begin{itemize}
    \item We introduce a dimension-corrected, stepwise measure of principal-subspace overlap and show that elevated overlap is a warning signal associated with declining validation performance in rollout RL.
    \item We formulate rollout RL as a constrained policy optimization problem and propose GCPO to enforce bilateral orthogonality on every effective policy update, which preserves the principal pretrained mappings while retaining a large complementary update space.
    \item Across two model families and three task domains, GCPO outperforms all evaluated baselines, better preserves cross-task capabilities, and exhibits more stable accuracy, policy-entropy, and response-length dynamics.
\end{itemize}

\section{Related Work}

\paragraph{Rollout-based RL for LLM Alignment.}
% Focus on PPO, GRPO, SDPO, and their known stability issues.
Rollout-based reinforcement learning has emerged as the dominant paradigm for enhancing LLM reasoning capabilities. Building upon PPO~\cite{PPO}, recent advances like GRPO~\cite{GRPO}, GSPO~\cite{GSPO}, GMPO~\cite{GMPO}, and DAPO~\cite{DAPO} have significantly improved advantage estimation and scaling efficiency. Despite these innovations, the dynamic nature of self-generated rollouts frequently drives policies toward optimization instabilities, response-length inflation, and general capability degradation~\cite{Qwen3}. Crucially, existing methods primarily intervene through objectives or observable policy behavior, leaving the geometry of realized parameter updates less explored. This highlights the need to diagnose the roots of these instabilities from a novel scope rather than symptomatic output-space patching.

\paragraph{Regularization and Stability in RL.}
RL post-training is commonly stabilized through objective- or
policy-level controls. KL penalties discourage deviation from a
reference policy, clipping constrains large likelihood-ratio changes, and reward shaping adjusts the scalar optimization signal to reduce undesirable behaviors such as reward hacking or length inflation~\cite{singhal2023long, Length-Inflation}. These techniques are effective, but they remain soft controls: they discourage unstable updates rather than ruling them out. Their effectiveness often depends on carefully tuned penalty coefficients and can be weakened by the high-variance feedback of rollout-based RL~\cite{achiam2017constrained,grontas2025pinet}. These limitations motivate looking beyond objective-level controls and designing a complementary approach that constrains the feasible update space directly.

\paragraph{Geometric Analysis of Policy Optimization.}
Recent studies suggest that the geometry of parameter updates plays a critical role in RL-based post-training.  Rather than treating policy updates as unstructured perturbations, these studies show that RL-induced updates exhibit systematic geometric patterns relative to the pretrained weights. In particular, recent analyses find that successful RL updates are, on average, more concentrated outside the dominant singular subspaces of the pretrained operators~\cite{GeometryOPD2026, Foresee2026}. This off-principal tendency suggests that effective policy adaptation may preferentially exploit directions that interfere less with the dominant structures learned during pre-training, which provides a foundation for understanding the importance of parameter-space geometry in rollout-based RL. However, existing analyses mainly characterize aggregate update behavior, leaving open how stepwise deviations from this geometry emerge during training and whether they are related to optimization instability.

\section{Principal-Subspace Overlap and Instability}
\label{sec: diagnose}

\begin{figure*}[t]
    \centering
    \includegraphics[width=1\linewidth]{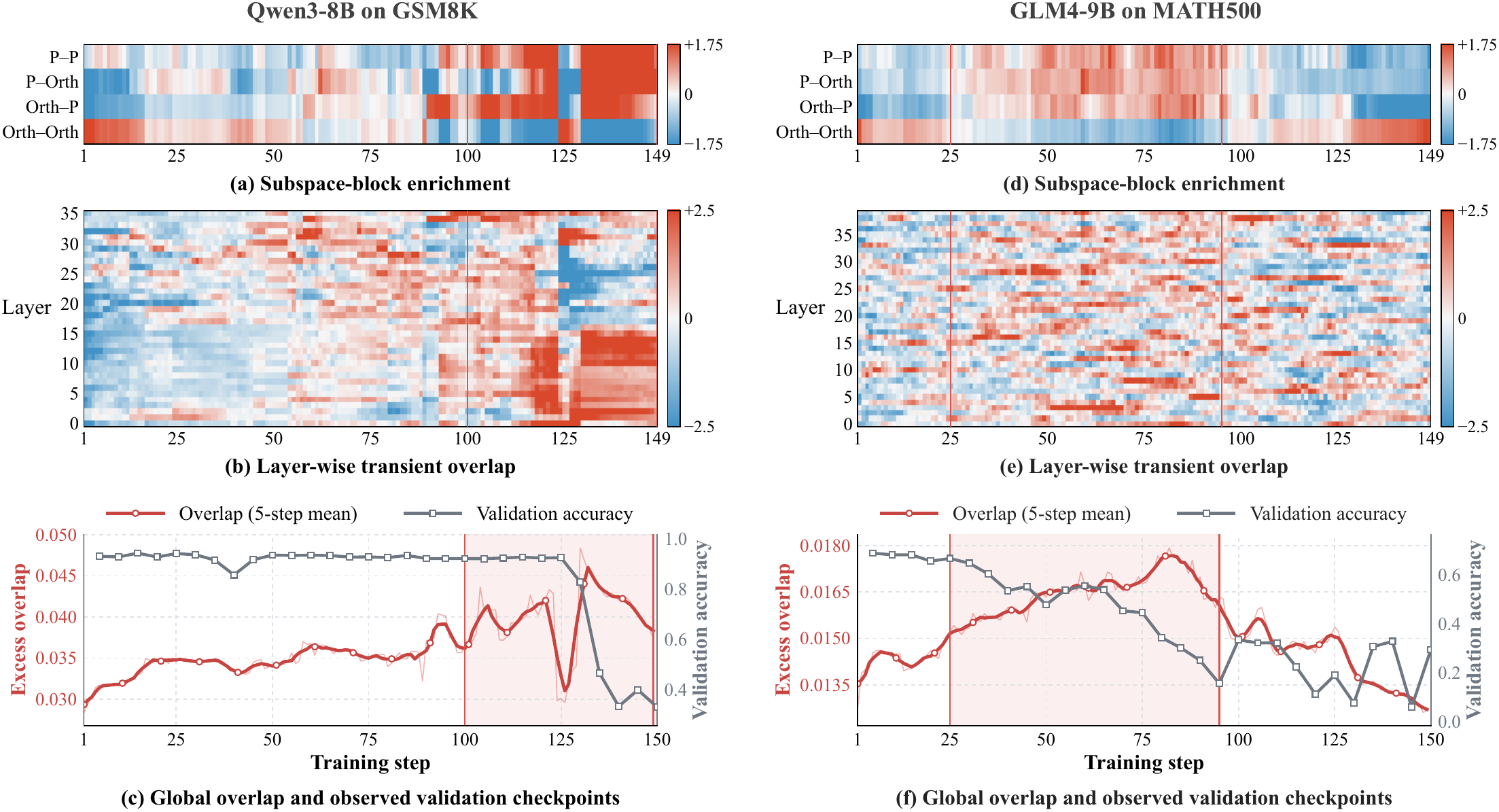}
    \caption{Stepwise update overlap and validation performance. The red curve is the 5-step moving average of excess principal-subspace overlap; the blue curve is validation accuracy. In both runs, episodes of elevated overlap accompany subsequent gradual degradation. This observation is correlational.
}
    \label{fig:intrusion_accuracy_coupling}
\end{figure*}

Prior work finds that RL updates, when aggregated over training, predominantly lie outside the dominant singular subspaces of pretrained weights~\cite{GeometryOPD2026, Foresee2026}. Nevertheless, aggregate
statistics may obscure intermittent spikes, termed \emph{excess principal-subspace overlap}, which often precede validation degradation. Controlled interventions further show that amplifying this overlap causes dose-dependent accuracy drops, motivating complementary-subspace constraints.

\subsection{Preliminaries}
For a pretrained weight matrix $W_{ref} \in \mathbb{R}^{d_{out} \times d_{in}}$, let $W_{ref}=\Phi \Sigma \Psi^\top$ be its SVD. We call the spans of the top-$k$ left and right singular vectors, $\Phi_k$ and $\Psi_k$, its \emph{principal subspaces}. They capture high-energy input--output directions of the pretrained linear operator; they are a structural proxy, not a literal partition of semantic knowledge. Let $\Pi_\Phi=\Phi_k\Phi_k^\top, \Pi_\Psi=\Psi_k\Psi_k^\top$ be the associated projectors, and $\Pi_\Phi^\perp=I-\Pi_\Phi, \Pi_\Psi^\perp=I-\Pi_\Psi$ denote the projectors onto their orthogonal complements.

We analyze the realized update $\delta^{(t)} W=W_t-W_{t-1}$, rather than the raw gradient, so the measurement includes the optimizer, learning rate, and momentum. Its energy can be decomposed into four mutually orthogonal blocks:
\begin{equation}
\begin{aligned}
    &\delta^{(t)} W = \underbrace{\Pi_\Phi \delta^{(t)}W \Pi_\Psi}_{\delta^{(t)}W^{PP}} + \underbrace{\Pi_\Phi \delta^{(t)}W \Pi_\Psi^\perp}_{\delta^{(t)} W^{PO}} \\
    & + \underbrace{\Pi_\Phi^\perp \delta^{(t)} W \Pi_\Psi}_{\delta^{(t)} W^{OP}} + \underbrace{\Pi_\Phi^\perp \delta^{(t)} W \Pi_\Psi^\perp}_{\delta^{(t)} W^{OO}}.
\end{aligned}
\end{equation}
Let $E_{ij}=\|\delta^{(t)} W^{ij}\|_F^2$ and $E_{\mathrm{total}}=\|\delta^{(t)} W\|_F^2$. The $OO$ block is orthogonal to the principal subspace on both sides; the other three blocks have overlap on at least one side. We summarize this overlap by

\begin{equation}
    O_t = \frac{E_{PP}+E_{PO}+E_{OP}}{E_{\mathrm{total}}}
        = 1-\frac{E_{OO}}{E_{\mathrm{total}}}.
\end{equation}

\subsection{Observed Excess Principal-subspace Overlap}
\label{sec:diagnose}

This raw ratio depends partly on dimensionality: an isotropic update has nonzero overlap simply because the principal subspaces have dimension $k$. For a $d_{\mathrm{out}} \times d_{\mathrm{in}}$ matrix, its expected overlap is
\begin{equation}
    O_{\mathrm{null}}
    =
    1-
    \frac{
        (d_{\mathrm{out}}-k)
        (d_{\mathrm{in}}-k)
    }{
        d_{\mathrm{out}}d_{\mathrm{in}}
    }.
\end{equation}
We therefore report the dimension-corrected quantity
\begin{equation}
    O_t^{\mathrm{excess}}=O_t-O_{\mathrm{null}}.
\end{equation}
Positive $O_t^{\mathrm{excess}}$ means that an update is more aligned with the principal subspaces than an isotropic update of the same size. We refer to this quantity as \emph{excess principal-subspace overlap}. It is an observable alignment statistic, not a claim that every principal direction is harmful.

Figure~\ref{fig:intrusion_accuracy_coupling} shows two consistent findings. First, most update energy remains in $OO$ (93.8\% for Qwen3-8B and 97.7\% for GLM4-9B on average), agreeing with the aggregate off-principal trend. Second, this average masks intermittent positive spikes in $O_t^{\mathrm{excess}}$. In the displayed runs, sustained or repeated spikes coincide with, and precede, declining validation accuracy; the effect appears as gradual deterioration in one setting and a sharp drop in another. Layer-wise measurements further show that these spikes are concentrated rather than uniform, often in intermediate and upper layers. More results are provided in Appendix~\ref{appendix:additional_overlap}.

\paragraph{Controlled intervention.}
To move beyond correlation, we intervene on a GRPO update of Qwen3-8B on ToolAlpaca. We rescale its principal-overlapping component by $\eta$ while preserving the Frobenius norm of every layer-wise update. Details are provided in Appendix~\ref{appendix:experimental_setup}. As shown in Figure~\ref{fig:controlled_intervention}, increasing overlap produces a
clear dose-dependent accuracy drop. Relative to the original update
($\eta=1$, 56.89\%), removing the overlapping component improves
accuracy, whereas principal-subspace injection substantially reduces accuracy. A matched random-subspace intervention causes substantially less degradation. These results provide local intervention evidence that increasing principal-subspace overlap can directly harm model performance, beyond the effects of update magnitude or arbitrary directional perturbation. Additional interventions across models, tasks, and checkpoints show the same dose-dependent trend (Appendix~\ref{appendix:intervention}).

\begin{figure}[t]
    \centering
    \includegraphics[width=\linewidth]{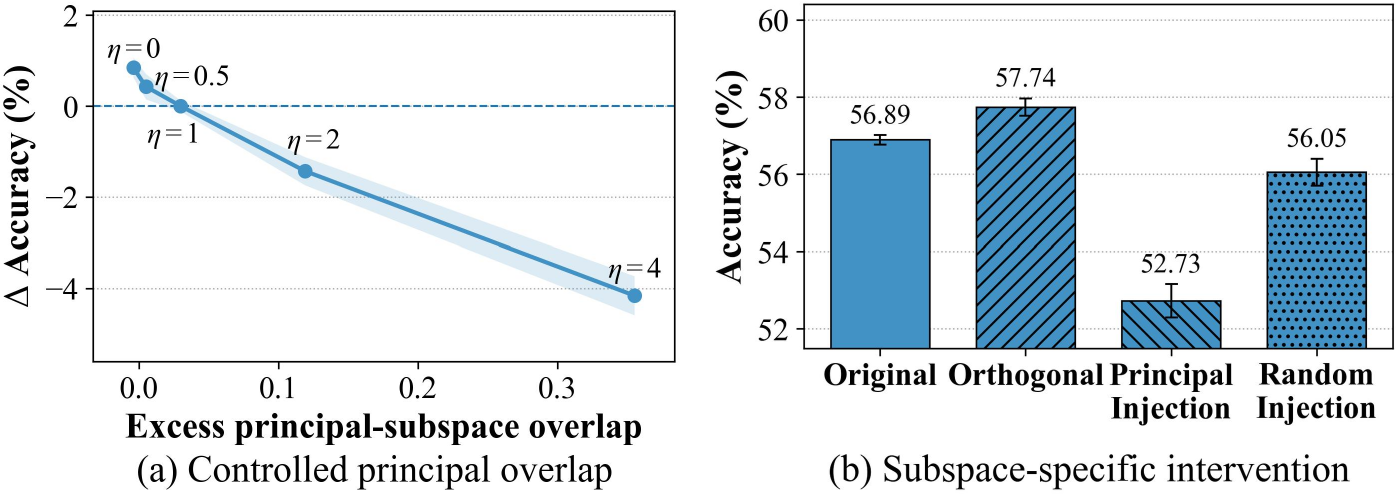}
    \caption{
    Controlled intervention on the step-150 GRPO update of Qwen3-8B on
    ToolAlpaca. (a) Increasing the principal-overlapping component under
    layer-wise norm matching produces a dose-dependent accuracy drop.
    (b) Orthogonalization improves accuracy, whereas matched
    principal-subspace injection is substantially more harmful than a
    random-subspace intervention.
    }
    \label{fig:controlled_intervention}
\end{figure}

Together, the observational and intervention results motivate preserving the dominant input and output directions while optimizing in their bilateral orthogonal complement. We implement this restriction next.

\section{Methodology}

Our diagnostic suggests a simple intervention: rather than modifying the rollout objective, restrict each policy update to the bilateral orthogonal complement of the dominant singular subspaces of the pretrained weights. GCPO realizes this intervention within GRPO through a projected low-rank parameterization.

\subsection{Problem Formulation}
\label{sec: problem formulation}

Let $\theta_0$ be the pretrained policy and $\mathcal{M}$ the set of adapted linear layers. For each $\ell\in\mathcal{M}$, we compute once the rank-$k$ singular subspaces $\Phi_k^{(\ell)}$ and $\Psi_k^{(\ell)}$
of $W_{\mathrm{ref}}^{(\ell)}$. We seek a policy update that maximizes the usual rollout objective while having zero principal-subspace overlap at every adapted layer.

\begin{align}
    \max_{\Delta \theta} \quad & \mathcal{J}_{\text{rollout}}(\pi_{\theta_0 + \Delta \theta}) - \beta \mathbb{D}_{\text{KL}}(\pi_{\theta_0 + \Delta \theta} \,||\, \pi_{\text{ref}}), \label{eq: grpo_objective} \\
    \text{s.t.} \quad & \Phi_k^{(\ell)\top} \delta^{(t)} W^{(\ell)}=0, \label{eq: constraint_1} \\
    & \delta^{(t)} W^{(\ell)}\Psi_k^{(\ell)}=0, \quad \forall \ell \in \mathcal{M}, \label{eq: constraint_2}
\end{align}
The two constraints are equivalent to retaining only the $OO$ block of each update. They complement, rather than replace, KL regularization: KL controls policy change in output space, whereas the constraints control where the parameter update can lie. A penalty could encourage these constraints, but would require trading them against the rollout objective. We instead enforce them exactly by parameterization.

\subsection{GCPO: Geometrically Constrained Policy Optimization}
\label{sec: main method}

For each adapted layer, with fixed orthogonal-complement projectors $\Pi_\Phi^\perp, \Pi_\Psi^\perp$, we can parameterize its update as
\begin{equation}
   \delta^{(t)} W^{(\ell)} = \alpha \Pi_\Phi^\perp L^{(\ell)} R^{(\ell)} \Pi_\Psi^\perp,
\end{equation}
where $L^{(\ell)}$ and $R^{(\ell)}$ are trainable low-rank factors and $\alpha$ is the standard scaling constant. The right projector removes components acting on dominant input directions; the left projector removes components written into dominant output directions.

This parameterization satisfies both constraints for every value of the trainable factors in Eq. (\ref{eq: constraint_1}) and Eq. (\ref{eq: constraint_2}). Thus, even if rollout gradients contain principal-subspace components, those components cannot change the effective layer update. This is the central property of GCPO: it converts the observed preference for off-principal updates into an exact feasible set, while leaving optimization within that set unchanged.

\subsection{Theoretical Properties and Comparison}
\label{sec: Theoretical analysis}

\paragraph{Exact subspace preservation.}
GCPO satisfies the bilateral constraints at every adapted layer and every optimization step. Consequently, the pretrained layer mapping is unchanged for inputs in the selected principal input subspace, with an analogous guarantee for the principal output subspace. Formal proofs are provided in Appendix~\ref{appendix:proofs}.

\paragraph{Capacity and cost.}
The feasible update space has dimension $(d_{\mathrm{out}}-k)(d_{\mathrm{in}}-k)$, and every feasible rank-$r$ update can be represented by the projected low-rank parameterization. Thus, for $k$ small relative to the layer width, GCPO retains a large space for task adaptation. Formal proofs are provided in Appendix~\ref{appendix:proofs}.

\paragraph{Distinction from Low-Rank Adaptation.}
Although \Method uses low-rank factorization for efficient
implementation, it solves a different constrained optimization
problem from conventional low-rank adaptation. LoRA primarily
restricts update rank for parameter efficiency, while
some geometric variants shape update directions to
prevent catastrophic forgetting typically under
fixed-data supervised fine-tuning~\cite{GeoLoRA, MiLoRA2025,
TailLoR2026, wang2023orthogonal}. In contrast, \Method is
motivated by transient \emph{principal-subspace overlap} under
on-policy rollout feedback in RL, where each policy update changes the
subsequent training distribution. It therefore constrains update
direction through bilateral orthogonality to stabilize stepwise
RL dynamics and preserve pretrained mappings, while low rank serves mainly
as an efficient parameterization of this feasible space.

\section{Experiments}
\label{sec:experiments}

In this section, we evaluate GCPO in terms of general-task performance, cross-task capability retention, training stability and efficiency, and design effectiveness through ablation and structural analyses.

% \subsection{Experimental Setup}
\label{sec:experimental_setup}

\paragraph{Models and Tasks.} We adopt two widely used instruction-tuned LLMs as our backbone models: Qwen3-8B \cite{Qwen3} and GLM4-9B \cite{GLM4}, and
evaluate them across three representative domains using established benchmarks and consistent protocols across methods:

\begin{itemize}
    \item \textbf{Mathematical Reasoning:} Evaluated on the MATH500 \cite{MATH500}. The reward is computed via exact matching of the final boxed answer using rule-based parsing.
    \item \textbf{Code Generation:} Evaluated on HumanEval+ \cite{HumanEval+} for Python function synthesis. The reward is granted based on the execution pass rate across unit tests.
    \item \textbf{Tool Use:} Evaluated on ToolAlpaca~\cite{ToolAlpaca}, which requires generating the appropriate API call given a tool specification and a user request. The reward is based on exact function-name matching, argument-key consistency, and normalized argument-value matching.
\end{itemize}

\paragraph{Baselines.}
We extensively compare GCPO against a comprehensive suite of optimization strategies. These include GRPO~\cite{GRPO}, objective-level variants including GSPO~\cite{GSPO}, DAPO~\cite{DAPO}, and GMPO~\cite{GMPO}, which modify objectives or clip probability ratios. To compare with parameter-space constraints, we include GRPO-LoRA~\cite{LoRA}, which uses unconstrained low-rank adaptation without the hard bilateral orthogonality enforced by GCPO. To isolate the effect of the directional constraint from low-rank parameterization, GRPO-LoRA and GCPO use the same adaptation rank and scaling configuration.

\paragraph{Implementation Details.}

For all GRPO-based methods, we sample $K=16$ rollouts per prompt during training. We adopt the same hyperparameters to ensure equity. For GCPO, we precompute the top-$k=8$ singular subspaces and constrain updates to their bilateral orthogonal complements. For evaluation, we report bootstrap-estimated majority@16 accuracy from 16 test-time responses per example. More details on seeds, data splits, evaluation, and hyperparameters are provided in Appendix~\ref{appendix:experimental_setup}.

\subsection{Main Results}
\label{sec:general_performance}

Table~\ref{table: main results of grpo} reports the mean and standard deviation over three independent training seeds across three task domains and two backbone models. GCPO achieves the best accuracy in all six model--task settings. Relative to the corresponding instruction-tuned base models, it improves accuracy by $7.09$--$27.69$ percentage points, and surpasses the strongest competing method on each benchmark by $1.02$--$2.37$ points. Moreover, GCPO exhibits the lowest standard deviation in all six settings, indicating reduced sensitivity to training stochasticity. The margins are particularly pronounced on GLM4-9B, where GCPO outperforms the strongest baselines by $2.15$--$2.37$ points, demonstrating consistent effectiveness across backbones with different optimization behaviors.

Beyond individual benchmarks, GCPO also achieves the strongest average performance. On Qwen3-8B, it obtains an average accuracy of $78.63$, exceeding the strongest baseline average of $77.48$ by $1.15$ points. On GLM4-9B, it improves the strongest baseline average from $73.73$ to
$76.12$. Averaged across all six model--task settings, GCPO reaches $77.38$, compared with $75.61$ for the strongest baseline, which suggests effectiveness and generalizability.

Crucially, we emphasize that these gains are not merely artifacts of low-rank parameter reduction: as shown in Table~\ref{table: main results of grpo}, GRPO-LoRA with the same rank achieves consistently lower accuracy than GCPO across all evaluations. This confirms that the explicit directional orthogonal projection is the primary driver of our method's superiority.

\begin{table}[h]
\centering
\fontsize{9}{11}\selectfont
\setlength{\tabcolsep}{3.5pt}
\begin{tabular}{@{}lccc@{}}
  \toprule[1.2pt]
    \textbf{Method}
    & \textbf{MATH500}
    & \textbf{HumanEval+}
    & \textbf{ToolAlpaca} \\
  \midrule[1.2pt]

  \multicolumn{4}{c}{\textit{Qwen3-8B}} \\
  \midrule
    Base (Instruct)
    & 67.46
    & 73.58
    & 56.53 \\

    GRPO
    & 72.00 {\scriptsize $\pm 1.36$}
    & 84.24 {\scriptsize $\pm 0.88$}
    & 59.56 {\scriptsize $\pm 1.52$} \\

    GSPO
    & 77.80 {\scriptsize $\pm 0.74$}
    & 87.81 {\scriptsize $\pm 0.51$}
    & \underline{66.18} {\scriptsize $\pm 0.82$} \\

    DAPO
    & \underline{78.33} {\scriptsize $\pm 0.61$}
    & 88.13 {\scriptsize $\pm 0.48$}
    & 65.99 {\scriptsize $\pm 0.76$} \\

    GMPO
    & 77.64 {\scriptsize $\pm 0.69$}
    & \underline{88.14} {\scriptsize $\pm 0.45$}
    & 66.16 {\scriptsize $\pm 0.71$} \\

    GRPO-LoRA
    & 77.87 {\scriptsize $\pm 0.58$}
    & 87.36 {\scriptsize $\pm 0.62$}
    & 66.05 {\scriptsize $\pm 0.67$} \\

    \rowcolor{gray!20}
    GCPO
    & \textbf{79.47} {\scriptsize $\pm 0.31$}
    & \textbf{89.16} {\scriptsize $\pm 0.27$}
    & \textbf{67.26} {\scriptsize $\pm 0.39$} \\

  \midrule
  \multicolumn{4}{c}{\textit{GLM4-9B}} \\
  \midrule
    Base (Instruct)
    & 66.51
    & 76.55
    & 42.47\\

    GRPO
    & 59.43 {\scriptsize $\pm 1.84$}
    & 72.43 {\scriptsize $\pm 1.42$}
    & 65.70 {\scriptsize $\pm 1.65$} \\

    GSPO
    & 71.29 {\scriptsize $\pm 0.97$}
    & 79.55 {\scriptsize $\pm 0.83$}
    & 66.22 {\scriptsize $\pm 1.08$} \\

    DAPO
    & \underline{72.41} {\scriptsize $\pm 0.82$}
    & 81.43 {\scriptsize $\pm 0.71$}
    & 67.35 {\scriptsize $\pm 0.93$} \\

    GMPO
    & 71.33 {\scriptsize $\pm 0.91$}
    & 80.63 {\scriptsize $\pm 0.76$}
    & \underline{67.79} {\scriptsize $\pm 0.88$} \\

    GRPO-LoRA
    & 72.34 {\scriptsize $\pm 0.66$}
    & \underline{81.48} {\scriptsize $\pm 0.59$}
    & 66.19 {\scriptsize $\pm 0.75$} \\

    \rowcolor{gray!20}
    GCPO
    & \textbf{74.56} {\scriptsize $\pm 0.34$}
    & \textbf{83.64} {\scriptsize $\pm 0.29$}
    & \textbf{70.16} {\scriptsize $\pm 0.41$} \\

  \bottomrule[1.2pt]
\end{tabular}
\caption{Main results across tasks. We report bootstrap-estimated majority@16 accuracy as mean $\pm$ standard deviation over three training seeds, in percentage points.}
\label{table: main results of grpo}
\end{table}

\subsection{Cross-Task Capability Preservation}
\label{sec:capability_retention}

Single-domain RL post-training may improve the target task while degrading other capabilities. To examine whether our geometric constraint better preserves such general capabilities, we post-train models on MATH500 and evaluate their cross-task capabilities on code generation (HumanEval+) and tool-use (ToolAlpaca) benchmarks.

\paragraph{Existing RL variants poorly retain tool-use capability.}
As shown in Table~\ref{tab:cross_task_retention}, GRPO and its variants generally exhibit substantial degradation on ToolAlpaca after MATH500 post-training. This drop is most severe for GRPO, with a large decline of $-14.97$ on GLM4-9B. These results suggest that conventional RL strategies are insufficient to preserve general capabilities beyond the reward-optimized domain.

\paragraph{Mathematical reasoning can transfer to code generation.} 
In contrast, RL training on math tasks has a milder effect on coding capabilities. Several variants even improve code generation performance, indicating positive transfer from mathematical reasoning to coding, likely because both tasks rely on step-by-step reasoning and symbolic problem solving. 

\paragraph{GCPO best preserves general capabilities.}
GCPO obtains the best Worst $\Delta$ scores on both Qwen3-8B and GLM4-9B, with gains of $+1.03$ and $+0.91$, while also improving HumanEval+ by $+3.99$ and $+5.88$, respectively. 
These results indicate that constraining updates away from dominant pretrained subspaces better balances task-specific RL adaptation with the preservation of general capabilities.
\begin{table}[t]
    \centering
    \fontsize{9}{11}\selectfont
    \begin{tabular}{lccc}
        \toprule[1.2pt]
        \textbf{Method}
        & \textbf{HumanEval+}
        & \textbf{ToolAlpaca}
        & \textbf{Worst $\Delta$} $\uparrow$ \\
        \midrule[1.2pt]

        \multicolumn{4}{c}{\textit{Qwen3-8B}} \\
        \midrule
        Base (Instruct)
        & 73.58
        & 56.53
        & 0.00 \\

        GRPO
        & 64.71 {\scriptsize $(-8.87)$}
        & 50.72 {\scriptsize $(-5.81)$}
        & $-8.87$ \\

        GSPO
        & \underline{76.82} {\scriptsize $(+3.24)$}
        & 53.41 {\scriptsize $(-3.12)$}
        & $-3.12$ \\

        DAPO
        & 75.54 {\scriptsize $(+1.96)$}
        & 53.09 {\scriptsize $(-3.44)$}
        & $-3.44$ \\

        GMPO
        & 74.31 {\scriptsize $(+0.73)$}
        & 52.90 {\scriptsize $(-3.63)$}
        & $-3.63$ \\

        GRPO-LoRA
        & 74.63 {\scriptsize $(+1.05)$}
        & \underline{56.69} {\scriptsize $(+0.16)$}
        & \underline{$+0.16$} \\

        \rowcolor{gray!20}
        GCPO
        & \textbf{77.57} {\scriptsize $(+3.99)$}
        & \textbf{57.56} {\scriptsize $(+1.03)$}
        & \textbf{$+1.03$} \\

        \midrule
        \multicolumn{4}{c}{\textit{GLM4-9B}} \\
        \midrule
        Base (Instruct)
        & 76.55
        & 42.47
        & 0.00 \\

        GRPO
        & 76.91 {\scriptsize $(+0.36)$}
        & 27.50 {\scriptsize $(-14.97)$}
        & $-14.97$ \\

        GSPO
        & \underline{81.58} {\scriptsize $(+5.03)$}
        & 41.29 {\scriptsize $(-1.18)$}
        & $-1.18$ \\

        DAPO
        & 81.39 {\scriptsize $(+4.84)$}
        & 37.68 {\scriptsize $(-4.79)$}
        & $-4.79$ \\

        GMPO
        & 79.76 {\scriptsize $(+3.21)$}
        & 41.33 {\scriptsize $(-1.14)$}
        & $-1.14$ \\

        GRPO-LoRA
        & 81.00 {\scriptsize $(+4.45)$}
        & \underline{43.14} {\scriptsize $(+0.67)$}
        & \underline{$+0.67$} \\

        \rowcolor{gray!20}
        GCPO
        & \textbf{82.43} {\scriptsize $(+5.88)$}
        & \textbf{43.38} {\scriptsize $(+0.91)$}
        & \textbf{$+0.91$} \\

        \bottomrule[1.2pt]
    \end{tabular}
    \caption{
        Cross-task capability retention after training on MATH500 and evaluated on  HumanEval+ and ToolAlpaca benchmarks. Values in parentheses denote absolute accuracy changes relative to the base model. Worst $\Delta$ denotes the smaller change across the two evaluation tasks.
    }
    \label{tab:cross_task_retention}
\end{table}

\subsection{Optimization Dynamics of GCPO}
\label{sec:stability_efficiency}

We investigate the underlying optimization dynamics of GCPO, explicitly focusing on its ability to stabilize the training trajectory, suppress shortcut learning behaviors, and maintain hardware efficiency.

\paragraph{Robust Training Stability.} 
Figure~\ref{fig:stability of Qwen tooluse} illustrates the accuracy trajectories evaluated every 5 steps during training on the ToolAlpaca for Qwen3-8B. GRPO exhibits severe performance oscillations and high-variance fluctuations, whereas GCPO yields a smooth, continuously increasing accuracy curve. This is consistent with our finding that our bilateral orthogonal constraint effectively acts as a structural barrier, preventing variance-heavy gradients from intruding into the principal subspace and ensuring stable, directional policy improvement.

\begin{figure}[t]
    \centering
    \includegraphics[width=1\linewidth]{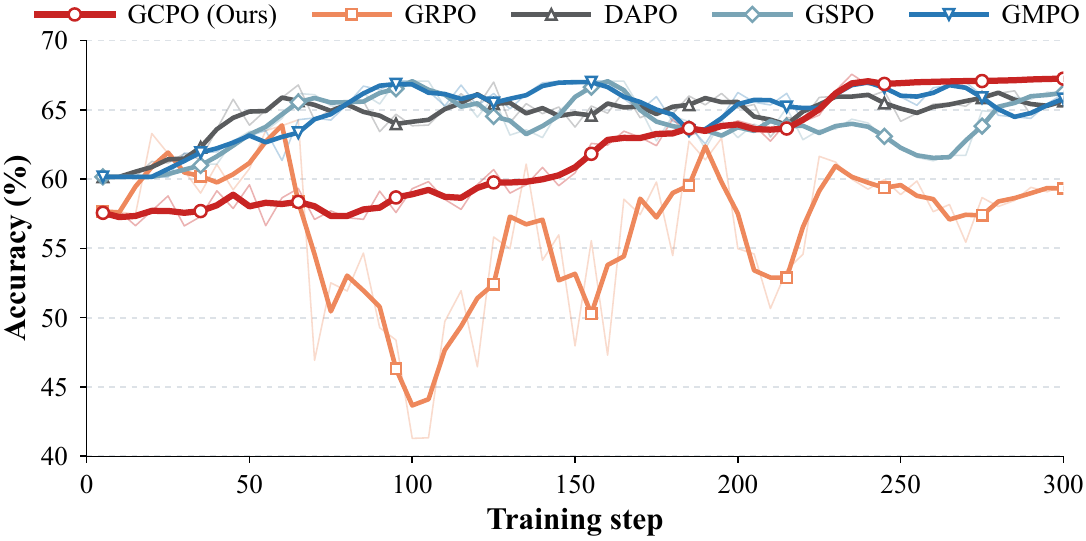}
    \caption{Training stability and accuracy trajectory on ToolAlpaca with Qwen3-8B, evaluated every 5 training steps.}
    \label{fig:stability of Qwen tooluse}
\end{figure}

\paragraph{Stable Policy Entropy Dynamics.}
We further examine the policy entropy during RL post-training to characterize the exploration--exploitation behavior of different methods. A desirable optimization trajectory should gradually reduce entropy, while avoiding both entropy oscillations and premature entropy collapse\cite{entropy,entropy2}. As shown in Figure~\ref{fig:policy_entropy_qwen}, GRPO exhibits large entropy oscillations, while most baselines exhibit rapid entropy decay, suggesting premature overconfident exploitation and reduced policy diversity. GCPO instead maintains a smooth, gradual decay, suggesting more controlled policy specialization without rapid loss of diversity.

\begin{figure}[t]
    \centering
    \includegraphics[width=1\linewidth]{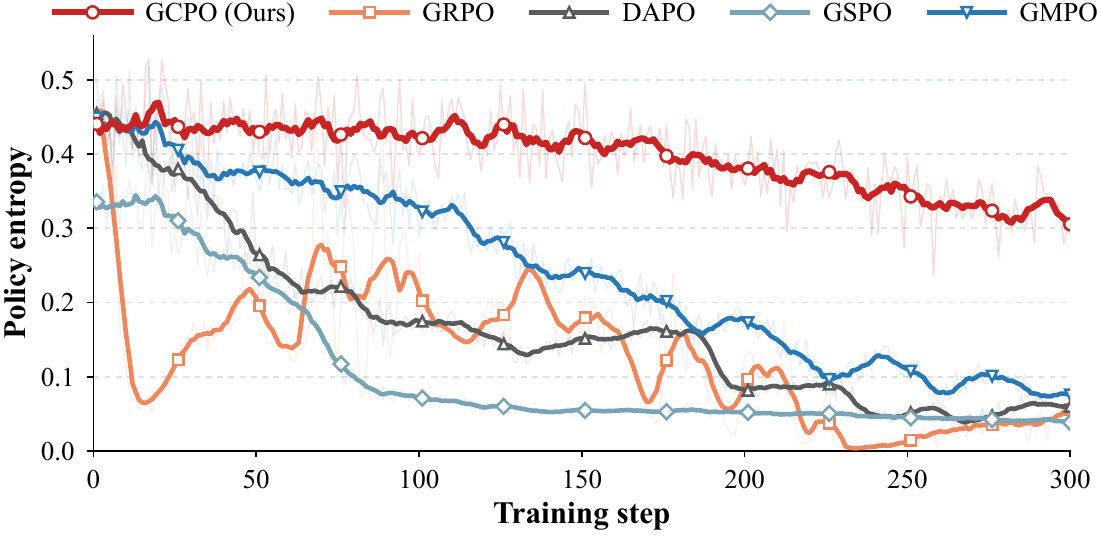}
    \caption{Policy entropy dynamics on MATH500 with GLM4-9B. GCPO maintains a smooth and gradual entropy decay, avoiding both the severe oscillations of GRPO and the premature entropy collapse in other baselines.}
    \label{fig:policy_entropy_qwen}
\end{figure}

\paragraph{Resistance to Response-Length Inflation.}
A notorious failure mode in rollout-based RL is \emph{response-length inflation}, where the policy learns a shortcut to hack rewards by generating excessively long, redundant tokens rather than improving genuine reasoning. As shown in Figure~\ref{fig:response_length_qwen_math}, GRPO suffers from aggressive and uncontrolled output length expansion on MATH500. Conversely, GCPO substantially reduces response length, maintaining concise and stable generation lengths. We hypothesize that shielding dominant pretrained directions—including those encoding length priors—from high-variance gradients reduces the policy's tendency to exploit verbosity as a reward-hacking shortcut.

\begin{figure}[t]
    \centering
    \includegraphics[width=1\linewidth]{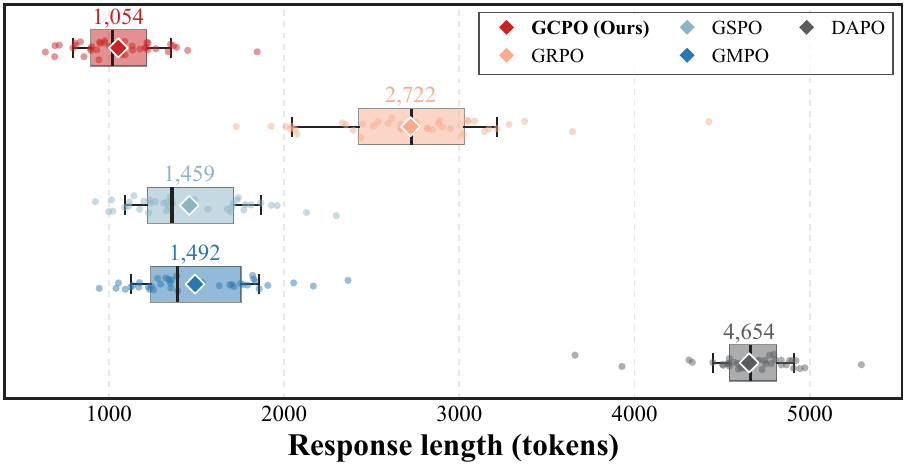}
    \caption{Response length on MATH500 with Qwen3-8B. GCPO most effectively mitigates response-length inflation.}
    \label{fig:response_length_qwen_math}
\end{figure}

\paragraph{Superior Memory Efficiency.}
Figure~\ref{fig:vram_profile_qwen_math} compares peak GPU memory under matched GRPO training configurations. Full-parameter GRPO requires substantially more memory because optimizer states and gradients are maintained for all adapted dense weights. Both GRPO-LoRA and GCPO reduce this cost through low-rank adaptation. Importantly, GCPO achieves comparable peak memory to LoRA, showing that the bilateral geometric constraint does not sacrifice the memory efficiency of parameter-efficient training. The principal singular vectors are frozen and require no optimizer states, while the learned update can be merged into the pretrained weights for inference.

\begin{figure}[t]
    \centering
    \includegraphics[width=1\linewidth]{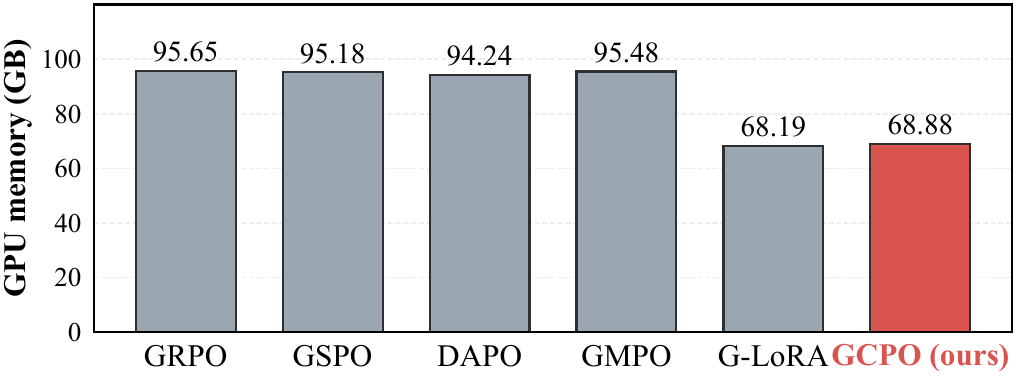}
    \caption{Peak per-GPU training memory on MATH500 with GLM4-9B, where G-LoRA abbreviates GRPO-LoRA.}
    \label{fig:vram_profile_qwen_math}
\end{figure}

\subsection{Ablation Studies}
\label{sec:ablation_analysis}

\paragraph{Orthogonal Constraints Prevent Capability Collapse.}
Table~\ref{tab:tooluse_geometric_ablation} reveals that the target subspace strongly affects optimization performance. Forcing updates into the \emph{Principal} subspace triggers a severe performance collapse, indicating that altering dominant singular directions disrupts task-relevant capabilities.
Conversely, the \emph{Orthogonal} complement achieves peak accuracy, substantially outperforming \emph{Random} projections. By structurally routing updates away from the core manifold, it safely resolves the conflict between aggressive task adaptation and knowledge retention.

\paragraph{Bilateral Projections Seal Collapse Pathways.}
Unilateral constraints on solely the \emph{left} (output) or \emph{right} (input) singular spaces yield only marginal gains over the unconstrained baseline. Geometrically, securing just one side leaves the opposite subspace vulnerable to dominant task gradients, permitting partial overlap. The \emph{bilateral} projection is mathematically necessary to seal both pathways simultaneously, ensuring zero principal-subspace overlap and optimal accuracy.

\paragraph{Hard Constraints Survive Iterative Rollout Dynamics.}
Both KL divergence and explicit soft orthogonality penalties ($\|\Phi_k^\top \delta W\|_F^2 + \|\delta W \Psi_k\|_F^2$) underperform our approach. Under the iterative feedback loops of RL, dominant task gradients inevitably overwhelm soft loss terms, allowing overlap to persist. By enforcing orthogonality by construction, our \emph{Hard} constraint robustly shields the pretrained manifold.

\begin{table}[t]
    \centering
    \small
    \begin{tabular}{llc}
        \toprule[1.2pt]
        \textbf{Ablation}
        & \textbf{Variant}
        & \textbf{MATH500 Acc.} $\uparrow$ \\
        \midrule[1.2pt]

        \multirow{4}{*}{Projection}
        & w/o. Constraint & 72.34 \\
        & Left-only               & 73.49 \\
        & Right-only              & 73.56 \\
        & \cellcolor{gray!15}\textbf{Bilateral (Ours)}
        & \cellcolor{gray!15}\textbf{74.56} \\
        \midrule

        \multirow{3}{*}{Subspace}
        & Random              & 66.47 \\
        & Principal           & 62.59 \\
        & \cellcolor{gray!15}\textbf{Orthogonal (Ours)}
        & \cellcolor{gray!15}\textbf{74.56} \\
        \midrule

        \multirow{3}{*}{Constraints}
        & Soft Loss Regularization & 71.11 \\
        & KL Regularization     & 67.83 \\
        & \cellcolor{gray!15}\textbf{Hard (Ours)}
        & \cellcolor{gray!15}\textbf{74.56} \\
        \bottomrule[1.2pt]
    \end{tabular}
    \caption{
        Geometric design ablations on MATH500 benchmark with GLM4-9B, considering different projection directions, subspace selections, and constraint mechanisms.
    }
    \label{tab:tooluse_geometric_ablation}
\end{table}

\paragraph{The Protected Dimension Balances Protection and Adaptation.}
Performance varies substantially with $k$ and peaks at $k=8$ on
both tasks (Figure~\ref{fig:k_ablation}). Smaller $k$ provides insufficient protection, whereas larger $k$ overly restricts the feasible update space, highlighting the importance of $k$ in balancing subspace protection and adaptation capacity.

\begin{figure}[t]
    \centering
    \includegraphics[width=1\linewidth]{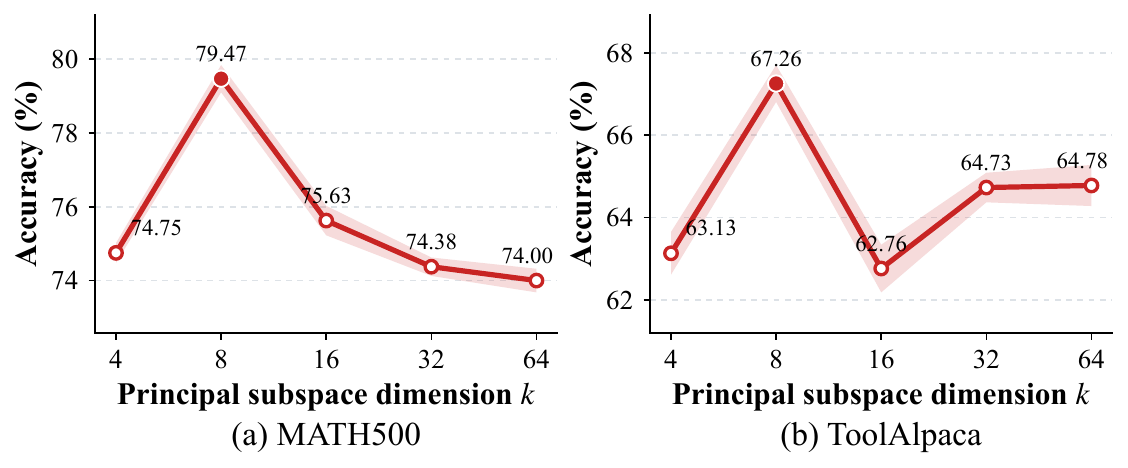}
    \caption{Effect of $k$ on MATH500 and ToolAlpaca with Qwen3-8B, reported as mean $\pm$ standard deviation.}
    \label{fig:k_ablation}
\end{figure}

\subsection{Limitations}
Our study focuses on the parameter-space dynamics of on-policy RL, and it remains unclear whether the observed geometric patterns also extend to other post-training paradigms, such as DPO~\cite{DPO}, KTO~\cite{KTO}, and OPD~\cite{OPD}. Although principal-subspace overlap is consistently associated with performance degradation, its causal relationship with model instability, rollout behaviors such as length inflation and reward hacking is not yet fully understood. Moreover, future work should examine broader alignment objectives, model scales, and training settings, while adaptive, layer-wise selection of $k$, to assess the generality of this principal-subspace overlap and develop more stable and capability-preserving post-training.

\section{Conclusion}
We study rollout-based RL instability through the geometry of step-wise parameter updates. We identify transient principal-subspace overlap as a warning signal associated with performance degradation, with controlled interventions providing local evidence of its harmful effect. Based on this diagnosis, we propose GCPO, which restricts effective updates to the bilateral orthogonal complements of pretrained principal subspaces. Across
two backbones and three task domains, GCPO achieves the highest mean performance overall, while improving capability retention, response-length control, and policy-entropy stability. These results highlight parameter-space geometry as a useful diagnostic and design perspective for stable LLM reinforcement learning.

\section*{Acknowledgement}
We gratefully acknowledge NovaCore for the valuable guidance and resources that supported this research. Additionally, We thank the open-source community for releasing code and datasets that made this research possible.

\bibliography{aaai2027}


% \newpage

% \iffalse

\clearpage

\appendix

\setcounter{secnumdepth}{1}

\section{Notations}
\label{appendix:notations}

Table~\ref{tab:notations} summarizes the main notation used in
our geometric analysis and method. The superscript $\ell$ denotes
a linear layer, and $t$ denotes an optimization step. We omit the
layer index when it is clear from context.

\begin{table*}[t]
    \centering
    \small
    \setlength{\tabcolsep}{6pt}
    \caption{Summary of the main notation.}
    \label{tab:notations}
    \begin{tabular}{@{}p{0.25\textwidth}p{0.69\textwidth}@{}}
        \toprule
        \textbf{Notation}
        & \textbf{Definition} \\
        \midrule

        $\theta_0,\theta_t$
        & Initial policy parameters and policy parameters after
        optimization step $t$. \\

        $\pi_{\theta_t},\pi_{\mathrm{ref}}$
        & Current policy and fixed reference policy. \\

        $\mathcal{M}$
        & Set of linear layers adapted during RL post-training. \\

        $W_{\mathrm{ref}}^{(\ell)},W_t^{(\ell)}$
        & Pretrained and adapted weight matrices of layer $\ell$. \\

        $W_{\mathrm{ref}}^{(\ell)}
        =\Phi^{(\ell)}\Sigma^{(\ell)}
        \Psi^{(\ell)\top}$
        & Singular value decomposition of the pretrained weight. \\

        $\Phi_k^{(\ell)},\Psi_k^{(\ell)}$
        & Top-$k$ left and right singular vectors, defining the
        principal output and input subspaces, respectively. \\

        $\Pi_{\Phi}^{(\ell)},
        \Pi_{\Psi}^{(\ell)}$
        & Projectors onto the principal subspaces:
        $\Pi_{\Phi}^{(\ell)}
        =\Phi_k^{(\ell)}\Phi_k^{(\ell)\top}$ and
        $\Pi_{\Psi}^{(\ell)}
        =\Psi_k^{(\ell)}\Psi_k^{(\ell)\top}$. \\

        $\Pi_{\Phi}^{(\ell)\perp},
        \Pi_{\Psi}^{(\ell)\perp}$
        & Projectors onto the corresponding orthogonal complements. \\

        $\Delta W_t^{(\ell)}
        =W_t^{(\ell)}-W_{\mathrm{ref}}^{(\ell)}$
        & Cumulative adaptation relative to the pretrained weight. \\

        $\delta^{(t)}W^{(\ell)}
        =W_t^{(\ell)}-W_{t-1}^{(\ell)}$
        & Realized update at optimization step $t$. \\

        $\delta^{(t)}W_{ij}^{(\ell)}$
        & Update block obtained by projecting its left and right
        sides onto $i,j\in\{P,O\}$, where $P$ and $O$ denote the
        principal and orthogonal-complement subspaces. \\

        $E_{ij}^{(t,\ell)}$
        & Squared Frobenius energy of block $ij$:
        $\|\delta^{(t)}W_{ij}^{(\ell)}\|_F^2$. \\

        $O_t$
        & Fraction of update energy overlapping at least one
        principal subspace. \\

        $O_{\mathrm{null}}$
        & Expected overlap of an isotropic update under the same
        matrix dimensions and protected rank. \\

        $O_t^{\mathrm{excess}}
        =O_t-O_{\mathrm{null}}$
        & Dimension-corrected principal-subspace overlap. \\

        $L^{(\ell)},R^{(\ell)}$
        & Trainable low-rank factors used by GCPO. \\

        $r,k,s$
        & Adaptation rank, protected principal rank, and adaptation
        scaling coefficient, respectively. \\

        $\|\cdot\|_F$
        & Frobenius norm. \\

        \bottomrule
    \end{tabular}
\end{table*}

\section{Proofs}
\label{appendix:proofs}

We provide proofs for the geometric properties used in the
diagnostic measure and GCPO formulation. For clarity, we omit the
layer index $\ell$ unless necessary.

\subsection{Orthogonality of the Update Decomposition}

\begin{lemma}
The four blocks
$\delta^{(t)}W^{PP}$,
$\delta^{(t)}W^{PO}$,
$\delta^{(t)}W^{OP}$, and
$\delta^{(t)}W^{OO}$
are mutually orthogonal under the Frobenius inner product.
Consequently,
\begin{equation}
    \|\delta^{(t)}W\|_F^2
    =
    E_{PP}+E_{PO}+E_{OP}+E_{OO}.
\end{equation}
\end{lemma}

\begin{proof}
The principal and complementary projectors satisfy
\begin{equation}
    \Pi_\Phi\Pi_\Phi^\perp=0,
    \qquad
    \Pi_\Psi\Pi_\Psi^\perp=0.
\end{equation}
Consider, for example, the $PP$ and $PO$ blocks. Their Frobenius
inner product is
\begin{align}
&\left\langle
\Pi_\Phi\delta^{(t)}W\Pi_\Psi,
\Pi_\Phi\delta^{(t)}W\Pi_\Psi^\perp
\right\rangle_F
\nonumber\\
&=
\operatorname{tr}\left(
\Pi_\Psi
\delta^{(t)}W^\top
\Pi_\Phi
\delta^{(t)}W
\Pi_\Psi^\perp
\right)
=0,
\end{align}
where the last equality follows from cyclic invariance of the trace
and $\Pi_\Psi^\perp\Pi_\Psi=0$. The remaining pairs follow
analogously from orthogonality of the left or right projectors.
The squared norm of their sum is therefore the sum of their
squared norms.
\end{proof}

It follows immediately that the fraction of energy overlapping at
least one principal subspace is
\begin{equation}
    O_t
    =
    \frac{E_{PP}+E_{PO}+E_{OP}}{E_{\mathrm{total}}}
    =
    1-\frac{E_{OO}}{E_{\mathrm{total}}}.
\end{equation}

\subsection{Expected Overlap of an Isotropic Update}

\begin{proposition}
Let $\delta W\in\mathbb{R}^{d_{\mathrm{out}}\times
d_{\mathrm{in}}}$ have a uniformly random direction under the
Frobenius norm. Then its expected principal-subspace overlap is
\begin{equation}
    \mathbb{E}[O]
    =
    1-
    \frac{(d_{\mathrm{out}}-k)(d_{\mathrm{in}}-k)}
    {d_{\mathrm{out}}d_{\mathrm{in}}}.
\end{equation}
\end{proposition}

\begin{proof}
Let
\begin{equation}
    z=\operatorname{vec}(\delta W)
    \in\mathbb{R}^{D},
    \qquad
    D=d_{\mathrm{out}}d_{\mathrm{in}}.
\end{equation}
Using the vectorization identity, the doubly orthogonal component
can be written as
\begin{equation}
    \operatorname{vec}
    \left(
    \Pi_\Phi^\perp\delta W\Pi_\Psi^\perp
    \right)
    =
    Qz,
    \qquad
    Q=
    \Pi_\Psi^\perp\otimes\Pi_\Phi^\perp.
\end{equation}
The matrix $Q$ is an orthogonal projector with rank
\begin{equation}
    \operatorname{rank}(Q)
    =
    (d_{\mathrm{out}}-k)(d_{\mathrm{in}}-k).
\end{equation}
For a uniformly random direction,
\begin{equation}
    \mathbb{E}
    \left[
    \frac{zz^\top}{\|z\|_2^2}
    \right]
    =
    \frac{I_D}{D}.
\end{equation}
Therefore,
\begin{align}
    \mathbb{E}
    \left[
    \frac{\|Qz\|_2^2}{\|z\|_2^2}
    \right]
    &=
    \operatorname{tr}
    \left(
    Q\,
    \mathbb{E}
    \left[
    \frac{zz^\top}{\|z\|_2^2}
    \right]
    \right) \\
    &=
    \frac{\operatorname{rank}(Q)}{D} \\
    &=
    \frac{(d_{\mathrm{out}}-k)(d_{\mathrm{in}}-k)}
    {d_{\mathrm{out}}d_{\mathrm{in}}}.
\end{align}
Since $O=1-\|Qz\|_2^2/\|z\|_2^2$, the stated result follows.
\end{proof}

This result gives the dimensionality-dependent null value
$O_{\mathrm{null}}$ used in the excess-overlap statistic.

\subsection{Exact Feasibility and Subspace Preservation}

\begin{proposition}
Suppose the cumulative adaptation of layer $\ell$ is parameterized
as
\begin{equation}
    \Delta W_t^{(\ell)}
    =
    s\Pi_\Phi^{(\ell)\perp}
    L_t^{(\ell)}R_t^{(\ell)}
    \Pi_\Psi^{(\ell)\perp}.
\end{equation}
Then both the cumulative adaptation and every realized update
satisfy the bilateral constraints:
\begin{equation}
    \Phi_k^{(\ell)\top}\Delta W_t^{(\ell)}=0,
    \qquad
    \Delta W_t^{(\ell)}\Psi_k^{(\ell)}=0,
\end{equation}
and
\begin{equation}
    \Phi_k^{(\ell)\top}\delta^{(t)}W^{(\ell)}=0,
    \qquad
    \delta^{(t)}W^{(\ell)}\Psi_k^{(\ell)}=0.
\end{equation}
\end{proposition}

\begin{proof}
By construction,
\begin{equation}
    \Phi_k^{(\ell)\top}\Pi_\Phi^{(\ell)\perp}=0,
    \qquad
    \Pi_\Psi^{(\ell)\perp}\Psi_k^{(\ell)}=0.
\end{equation}
Hence,
\begin{align}
    \Phi_k^{(\ell)\top}\Delta W_t^{(\ell)}
    &=
    s\Phi_k^{(\ell)\top}
    \Pi_\Phi^{(\ell)\perp}
    L_t^{(\ell)}R_t^{(\ell)}
    \Pi_\Psi^{(\ell)\perp}
    =0,\\
    \Delta W_t^{(\ell)}\Psi_k^{(\ell)}
    &=
    s\Pi_\Phi^{(\ell)\perp}
    L_t^{(\ell)}R_t^{(\ell)}
    \Pi_\Psi^{(\ell)\perp}
    \Psi_k^{(\ell)}
    =0.
\end{align}

The realized update is
\begin{equation}
    \delta^{(t)}W^{(\ell)}
    =
    \Delta W_t^{(\ell)}-\Delta W_{t-1}^{(\ell)}.
\end{equation}
Because the constraints are linear and both cumulative adaptations
satisfy them, their difference satisfies them as well.
\end{proof}

\begin{corollary}
For every
$x\in\operatorname{span}(\Psi_k^{(\ell)})$ and
$y\in\operatorname{span}(\Phi_k^{(\ell)})$,
\begin{equation}
    W_t^{(\ell)}x
    =
    W_{\mathrm{ref}}^{(\ell)}x,
    \qquad
    y^\top W_t^{(\ell)}
    =
    y^\top W_{\mathrm{ref}}^{(\ell)}.
\end{equation}
\end{corollary}

\begin{proof}
Write $x=\Psi_k^{(\ell)}a$ and
$y=\Phi_k^{(\ell)}b$. The bilateral constraints imply
$\Delta W_t^{(\ell)}x=0$ and
$y^\top\Delta W_t^{(\ell)}=0$. Substituting
$W_t^{(\ell)}=W_{\mathrm{ref}}^{(\ell)}
+\Delta W_t^{(\ell)}$ proves the result.
\end{proof}

The corollary is a layer-level algebraic guarantee for the selected
subspaces; it does not imply unconditional preservation of all
model-level capabilities.

\subsection{Dimension and Expressivity of the Feasible Space}

\begin{proposition}
Define the bilateral feasible space
\begin{equation}
    \mathcal{S}
    =
    \left\{
    A:
    \Phi_k^\top A=0,\;
    A\Psi_k=0
    \right\}.
\end{equation}
Then
\begin{equation}
    \mathcal{S}
    =
    \left\{
    \Pi_\Phi^\perp Z\Pi_\Psi^\perp:
    Z\in\mathbb{R}^{d_{\mathrm{out}}\times d_{\mathrm{in}}}
    \right\},
\end{equation}
and
\begin{equation}
    \dim(\mathcal{S})
    =
    (d_{\mathrm{out}}-k)(d_{\mathrm{in}}-k).
\end{equation}
Moreover, every matrix in $\mathcal{S}$ with rank at most $r$
can be represented by the GCPO low-rank parameterization.
\end{proposition}

\begin{proof}
If $A\in\mathcal{S}$, then its columns lie in
$\operatorname{span}(\Phi_k)^\perp$, implying
$\Pi_\Phi^\perp A=A$. Similarly, $A\Psi_k=0$ implies
$A\Pi_\Psi^\perp=A$. Thus,
\begin{equation}
    A=\Pi_\Phi^\perp A\Pi_\Psi^\perp.
\end{equation}

Let $\Phi_\perp$ and $\Psi_\perp$ be orthonormal bases for the
corresponding complementary subspaces. Every feasible matrix has
a unique representation
\begin{equation}
    A=\Phi_\perp B\Psi_\perp^\top,
    \qquad
    B\in
    \mathbb{R}^{(d_{\mathrm{out}}-k)
    \times(d_{\mathrm{in}}-k)}.
\end{equation}
Therefore, the dimension of $\mathcal{S}$ is
$(d_{\mathrm{out}}-k)(d_{\mathrm{in}}-k)$.

If $\operatorname{rank}(A)\leq r$, then
$\operatorname{rank}(B)\leq r$, so $B$ admits a factorization
$B=UV$ with at most $r$ latent dimensions. Taking
\begin{equation}
    L=\Phi_\perp U,
    \qquad
    R=V\Psi_\perp^\top
\end{equation}
gives
\begin{equation}
    A
    =
    \Pi_\Phi^\perp LR\Pi_\Psi^\perp,
\end{equation}
which is exactly the GCPO parameterization up to its scaling
coefficient.
\end{proof}

\section{Method Details}

\subsection{Global Overlap Aggregation and Robustness}
\label{appendix:global_overlap}
\paragraph{Global aggregation.}
For each adapted matrix $\ell\in\mathcal{M}$, we first compute
the layer-wise dimension-corrected overlap
\begin{equation}
O_{t,\ell}^{\mathrm{excess}}
=
\left(
1-
\frac{E_{OO}^{(t,\ell)}}
     {E_{\mathrm{total}}^{(t,\ell)}}
\right)
-
\left[
1-
\frac{
(d_{\mathrm{out}}^{(\ell)}-k_\ell)
(d_{\mathrm{in}}^{(\ell)}-k_\ell)
}{
d_{\mathrm{out}}^{(\ell)}
d_{\mathrm{in}}^{(\ell)}
}
\right].
\end{equation}
The global statistic shown in Figure~\ref{fig:intrusion_accuracy_coupling} is the unweighted macro-average
\begin{equation}
O_{t,\mathrm{global}}^{\mathrm{excess}}
=
\frac{1}{|\mathcal{M}|}
\sum_{\ell\in\mathcal{M}}
O_{t,\ell}^{\mathrm{excess}}.
\end{equation}
Normalization and null correction are therefore performed
within each matrix before aggregation. Consequently, each
adapted matrix contributes equally, preventing matrices with
larger dimensions or update norms from mechanically
dominating the global statistic. The displayed curve applies
a five-step moving average only after this layer-wise
aggregation.

\subsection{Algorithm Details}
\label{appendix:method}

Algorithm~\ref{alg:gcpo} summarizes the implementation of GCPO. The principal subspaces are computed once before training, while the projected low-rank factors are used for both policy optimization and rollout generation.

\begin{algorithm}[tb]
\caption{Geometrically Constrained Policy Optimization}
\label{alg:gcpo}
\textbf{Input}: Frozen policy $\theta_0$ and adapted layers
$\mathcal{M}$\\
\textbf{Parameter}: Protected rank $k$, adaptation rank $r$,
scaling $\alpha$, and training steps $T$\\
\textbf{Output}: Trained factors
$\{L^{(\ell)},R^{(\ell)}\}_{\ell\in\mathcal{M}}$
\begin{algorithmic}[1]
\STATE Set $s\leftarrow\alpha/r$.
\FOR{each layer $\ell\in\mathcal{M}$}
    \STATE Compute the top-$k$ singular vectors
    $\Phi_k^{(\ell)}$ and $\Psi_k^{(\ell)}$ of
    $W_{\mathrm{ref}}^{(\ell)}$.
    \STATE Initialize $L^{(\ell)}$ to zero and
    $R^{(\ell)}$ with the standard LoRA initialization.
\ENDFOR
\FOR{$t=1,\ldots,T$}
    \FOR{each layer $\ell\in\mathcal{M}$}
        \STATE $\bar L^{(\ell)}
        \leftarrow L^{(\ell)}
        -\Phi_k^{(\ell)}
        \bigl(\Phi_k^{(\ell)\top}L^{(\ell)}\bigr)$.
        \STATE $\bar R^{(\ell)}
        \leftarrow R^{(\ell)}
        -\bigl(R^{(\ell)}\Psi_k^{(\ell)}\bigr)
        \Psi_k^{(\ell)\top}$.
        \STATE Set
        $W_t^{(\ell)}
        \leftarrow W_{\mathrm{ref}}^{(\ell)}
        +s\bar L^{(\ell)}\bar R^{(\ell)}$.
    \ENDFOR
    \STATE Synchronize the projected factors to the rollout engine.
    \STATE Generate rollouts and compute the policy-optimization loss.
    \STATE Update the raw factors
    $\{L^{(\ell)},R^{(\ell)}\}_{\ell\in\mathcal{M}}$.
\ENDFOR
\STATE \textbf{return}
$\{L^{(\ell)},R^{(\ell)}\}_{\ell\in\mathcal{M}}$.
\end{algorithmic}
\end{algorithm}

\section{Experimental Details}
\label{appendix:experimental_setup}

\subsection{Datasets and Splits}
\label{appendix:data_splits}

Table~\ref{tab:dataset_splits} summarizes the datasets used for
RL training and held-out evaluation. MATH500 and HumanEval+ are
established benchmarks for mathematical reasoning and code
generation, respectively. Following the common practice of
constructing fixed train--test splits for controlled RL
evaluation~\cite{SDPO}, we build task-specific splits from their
official evaluation sets and keep them identical across all
methods. Consequently, the corresponding results measure
performance on our held-out subsets and should not be interpreted
as official full-benchmark scores.

\begin{table}[tbp]
    \centering
    \caption{Dataset splits used for RL training and held-out
    evaluation.}
    \label{tab:dataset_splits}
    \begin{tabular}{lcccl}
        \toprule
        \textbf{Task}
        & \textbf{Total}
        & \textbf{Train}
        & \textbf{Test}
        & \textbf{Protocol} \\
        \midrule
        MATH500
        & 500
        & 450
        & 50
        & 9:1, seed 42. \\

        HumanEval+
        & 164
        & 147
        & 17
        & 9:1, seed 42. \\

        ToolAlpaca
        & 4,114
        & 4,046
        & 68
        & 407:10 families \\
        \bottomrule
    \end{tabular}
\end{table}

For MATH500~\cite{MATH500} and HumanEval+~\cite{HumanEval+},
no problem identifier appears in both the training and evaluation
splits. Each HumanEval+ problem retains its complete EvalPlus
functional-test suite; tests associated with the same programming
problem are never divided between the two splits.

For ToolAlpaca, constructed from ToolAlpaca~\cite{ToolAlpaca}, the
held-out set contains examples from 10 API families that do not
appear during training: Axolotl, Auth0, A B\'iblia Digital,
Apache Superset, Am\'ethyste, Abstract Public Holidays,
AbuseIPDB, 1Forge, Lob.com, and AniAPI. This protocol evaluates
generalization to previously unseen tool families rather than to
new examples from APIs observed during training.

The MATH500 and HumanEval+ splits are generated once with seed 42
and kept fixed across all methods, backbone models, and training
runs. This data-splitting seed is independent of the three random
seeds used for repeated training. Therefore, the reported
variation across runs reflects training and sampling stochasticity,
rather than changes in the train--test partition.

\subsection{Training and Evaluation Settings}
\label{appendix:training_settings}

Table~\ref{tab:shared_training_settings} summarizes the settings
shared across tasks and methods unless overridden by a
method-specific configuration. All main model--task--method
configurations are independently trained with three random seeds.
For each configuration, we report the mean and standard deviation
of the three resulting task-level scores.

\begin{table}[tbp]
    \centering
    \small
    \setlength{\tabcolsep}{4pt}
    \caption{Shared training, rollout, and evaluation settings.}
    \label{tab:shared_training_settings}
    \begin{tabular}{@{}
        >{\raggedright\arraybackslash}p{0.42\columnwidth}
        >{\raggedright\arraybackslash}p{0.52\columnwidth}
        @{}}
        \toprule
        \textbf{Configuration}
        & \textbf{Setting} \\
        \midrule
        Train batch size
        & 32 prompts per optimization step \\

        Per-GPU micro-batch size
        & 1 \\

        Optimizer
        & AdamW; gradient clipping at 1.0 \\

        Data order
        & Shuffled during training \\

        Training precision
        & \texttt{bfloat16} with Fully Sharded Data Parallel
        (FSDP) \\

        Rollout engine
        & vLLM with tensor parallelism of 2 \\

        Maximum prompt length
        & 2,048 tokens \\

        Maximum response length
        & 8,192 tokens \\

        Training decoding
        & Temperature 1.0, top-$p=1.0$, and top-$k=-1$ \\

        Evaluation decoding
        & Temperature 0.6 and top-$p=0.95$ \\

        Evaluation samples
        & 16 sampled responses per held-out example \\

        Evaluation metric
        & Majority@16 with 1,000 bootstrap resamples \\

        Evaluation interval
        & Every 5 optimization steps \\

        Checkpoint interval
        & Every 50 steps by default; every 100 steps for DAPO \\

        Checkpoint retention
        & At most one checkpoint retained per run \\

        Chat template
        & Model-specific template; thinking mode is disabled
        for Qwen3 \\

        Independent runs
        & 3 random seeds per main configuration \\

        Result reporting
        & Mean $\pm$ standard deviation over the 3 independent runs \\
        \bottomrule
    \end{tabular}
\end{table}

For each training seed, we generate 16 responses per held-out example and report majority@16 to reduce sensitivity to individual stochastic generations and evaluate the model's consistency across multiple samples. We estimate majority@16 using 1,000 bootstrap trials, each of which resamples 16 responses with replacement and scores the majority-voted answer. Correctness is averaged over bootstrap trials and examples to obtain one score per training run, and final results are reported as the mean and standard deviation over three independent runs.

\paragraph{Effective Training Horizons.}
The common training loop uses \texttt{drop\_last=True} and a
default limit of 30 epochs. As a result, the effective number of
optimization steps depends on the size of the corresponding
training split. Table~\ref{tab:effective_training_steps}
summarizes the resulting training horizons.

\begin{table}[tbp]
    \centering
    \caption{Effective training horizons. All runs are configured
    for at most 300 steps, with a batch size of 32,
    \texttt{drop\_last=True}, and at most 30 epochs.}
    \label{tab:effective_training_steps}
    \begin{tabular}{@{}lcccc@{}}
        \toprule
        \textbf{Task}
        & \shortstack{\textbf{Train}\\\textbf{size}}
        & \shortstack{\textbf{Steps}\\\textbf{/ epoch}}
        & \shortstack{\textbf{30-epoch}\\\textbf{cap}}
        & \shortstack{\textbf{Effective}\\\textbf{steps}} \\
        \midrule
        MATH500
        & 450
        & 14
        & 420
        & 300 \\

        HumanEval+
        & 147
        & 4
        & 120
        & 120 \\

        ToolAlpaca
        & 4,046
        & 126
        & 3,780
        & 300 \\
        \bottomrule
    \end{tabular}
\end{table}

MATH500 and ToolAlpaca can reach the configured limit of 300
optimization steps. HumanEval+ contains 147 training problems,
yielding four complete batches per epoch and therefore at most
120 steps under the default 30-epoch limit. For cross-task evaluation, we use the same task-specific checkpointfor every method and seed: step 250 for MATH500, step 100 forHumanEval+, and step 150 for ToolAlpaca. We do not select a separate best-validation checkpoint for each seed.

\subsection{Method-specific Hyperparameters}
\label{appendix:method_hyperparameters}

We retain method-specific optimization settings rather than
forcing all algorithms to share a common rollout count, learning
rate, or clipping rule. Table~\ref{tab:method_hyperparameters}
lists the principal hyperparameters. The same configuration is
used across the three independent runs of each method.

\begin{table}[tbp]
    \centering
    \caption{Method-specific rollout and optimization
    hyperparameters.}
    \label{tab:method_hyperparameters}
    \begin{tabular}{@{}lccp{0.42\columnwidth}@{}}
        \toprule
        \textbf{Method}
        & \shortstack{\textbf{Roll-}\\\textbf{outs}}
        & \shortstack{\textbf{Learning}\\\textbf{rate}}
        & \textbf{Principal settings} \\
        \midrule
        GRPO
        & 16
        & $1\times10^{-5}$
        & Mini-batch size 32; clipping ratio 0.2. \\

        GSPO
        & 16
        & $1\times10^{-6}$
        & Lower and upper clipping thresholds of
        $3\times10^{-4}$ and $4\times10^{-4}$. \\

        GMPO
        & 16
        & $1\times10^{-6}$
        & lower and upper clipping thresholds of 0.4. \\

        DAPO
        & 16
        & $1\times10^{-6}$
        & Train batch size 32; generation batch size 64. \\
        \bottomrule
    \end{tabular}
\end{table}

The parameter-efficient configurations are summarized in
Table~\ref{tab:peft_settings}. LoRA and GCPO adapt all linear
layers. For matched comparisons, the full-parameter, LoRA, and
GCPO variants use identical data splits, batch sizes, rollout
settings, and evaluation protocols; they differ only in the
parameterization of the policy update.

\begin{table}[tbp]
    \centering
    \caption{Parameter-efficient adaptation settings used in the
    main experiments.}
    \label{tab:peft_settings}
    \begin{tabular}{@{}lcccc@{}}
        \toprule
        \textbf{Variant}
        & \textbf{Modules}
        & \textbf{Rank $r$}
        & \textbf{$\alpha$}
        & \shortstack{\textbf{Protected} \textbf{rank $k$}} \\
        \midrule
        LoRA
        & All linear
        & 32
        & 16
        & -- \\

        GCPO
        & All linear
        & 32
        & 16
        & 8 \\
        \bottomrule
    \end{tabular}
\end{table}

The main GCPO configuration uses
scaling parameter $\alpha=16$, and protects the top $k=8$
left and right singular directions of each adapted weight matrix.
The ablation studies additionally consider
$k\in\{4,8,16,32,64\}$ and adaptation ranks
$r\in\{16,32,64\}$.

\subsection{Reward and Evaluation Protocols}
\label{appendix:reward_protocols}

All three tasks use binary response-level rewards.
Table~\ref{tab:reward_protocols} summarizes the task-specific prediction extraction and correctness criteria. A response receives a binary reward of 1 if it satisfies the task-specific criterion and 0 otherwise.

\begin{table}[tbp]
    \centering
    \small
    \setlength{\tabcolsep}{3pt}
    \caption{Task-specific response extraction and scoring.}
    \label{tab:reward_protocols}
    \begin{tabular}{@{}p{0.19\columnwidth}
                        p{0.30\columnwidth}
                        p{0.43\columnwidth}@{}}
        \toprule
        \textbf{Task}
        & \textbf{Extraction}
        & \textbf{Correctness} \\
        \midrule

        MATH500
        & Final \texttt{\textbackslash boxed\{\}} expression
        in the last 100 characters
        & Normalized string matching with Math-Verify
        equivalence checking \\

        HumanEval+
        & Longest Markdown code block
        & Pass all EvalPlus functional tests within 5 seconds \\

        ToolAlpaca
        & Action names and input JSON
        & Unordered action-name match and merged argument-JSON
        match; no tool execution \\

        \bottomrule
    \end{tabular}
\end{table}

For MATH500, the model is prompted to provide its final answer
inside \texttt{\textbackslash boxed\{\}}. For HumanEval+, each
problem retains its full functional-test suite, and a generated
program is considered correct only when it passes every test.
The 5-second execution limit is used by the evaluator, regardless
of the shorter time limit stated in the textual prompt.

The ToolAlpaca evaluator measures static structural correctness
rather than actual API execution. It does not evaluate the order
of multiple actions, and arguments from multiple predicted actions
are merged before comparison. Accordingly, we describe this metric
as action-name and argument matching rather than execution
correctness.

For every training seed, each held-out example is evaluated using
16 sampled responses under the decoding configuration in
Table~\ref{tab:shared_training_settings}. The resulting responses
are aggregated into a single task-level score for that seed.
Final results are reported as the mean and standard deviation
across the three independently trained seeds.

\subsection{Computational Resources}
\label{appendix: resource}

All experiments were conducted on a single compute node equipped
with four NVIDIA A100 GPUs. Table~\ref{tab:computational_resources}
summarizes the hardware configuration and distributed training
setup. The 48-hour allocation denotes the maximum wall-clock
budget assigned to each job, rather than the actual runtime of
every training run.

\begin{table}[tbp]
    \centering
    \caption{Computational resources and distributed training
    configuration used in our experiments.}
    \label{tab:computational_resources}
    \begin{tabular}{@{}p{0.34\columnwidth}p{0.57\columnwidth}@{}}
        \toprule
        \textbf{Item} & \textbf{Configuration} \\
        \midrule
        Compute node
        & Single node \\
        
        GPU
        & $4\times$ NVIDIA A100 80\,GB \\
        
        CPU
        & 32 CPU cores \\
        
        Job time limit
        & 48 hours per submitted job \\
        
        Training precision
        & \texttt{bfloat16} \\
        
        Distributed training
        & Fully Sharded Data Parallel (FSDP) \\
        
        Rollout engine
        & vLLM \\
        
        Rollout parallelism
        & Tensor parallelism with $\mathrm{TP}=2$ \\
        
        Micro-batch size
        & 1 sample per GPU \\
        \bottomrule
    \end{tabular}
\end{table}

For the matched GRPO comparison, full-parameter GRPO, GRPO-LoRA, and GCPO use identical data splits, batch sizes, rollout configurations, and evaluation protocols. They differ only in the parameterization and geometric constraint of the policy update.

The effective number of optimization steps additionally depends
on the dataset size because the common trainer uses
\texttt{drop\_last=True} and terminates after 30 epochs.
Table~\ref{tab:effective_training_steps} reports the resulting
training horizon for each task.

For MATH500 and ToolAlpaca, the epoch limit permits the trainer to
reach the configured maximum of 300 steps. For HumanEval+, the
147-example training split produces only four complete batches
per epoch, resulting in at most 120 optimization steps. Since
checkpoints are saved every 50 steps, we use the step-100
checkpoint for cross-task evaluation. DAPO uses a separate epoch
budget of 3,000 and is therefore not subject to the default
120-step limit on HumanEval+.

\section{Additional Experiments}
\label{appendix: additional_exp}

\subsection{Additional Overlap Trajectories}
\label{appendix:additional_overlap}

Figure~\ref{fig:additional_overlap} complements the
representative cases in Figure~\ref{fig:intrusion_accuracy_coupling}
with additional trajectories on HumanEval+ and ToolAlpaca.
We use the same excess-overlap measure, 5-step moving
average, and validation protocol as in Section~\ref{sec:diagnose}.
Elevated overlap again tends to accompany or precede
validation degradation, although its timing and magnitude
vary across tasks. To avoid redundant per-seed plots, we
show one fixed-seed trajectory for each remaining task.

\begin{figure*}[t]
    \centering
    \includegraphics[width=1\linewidth]{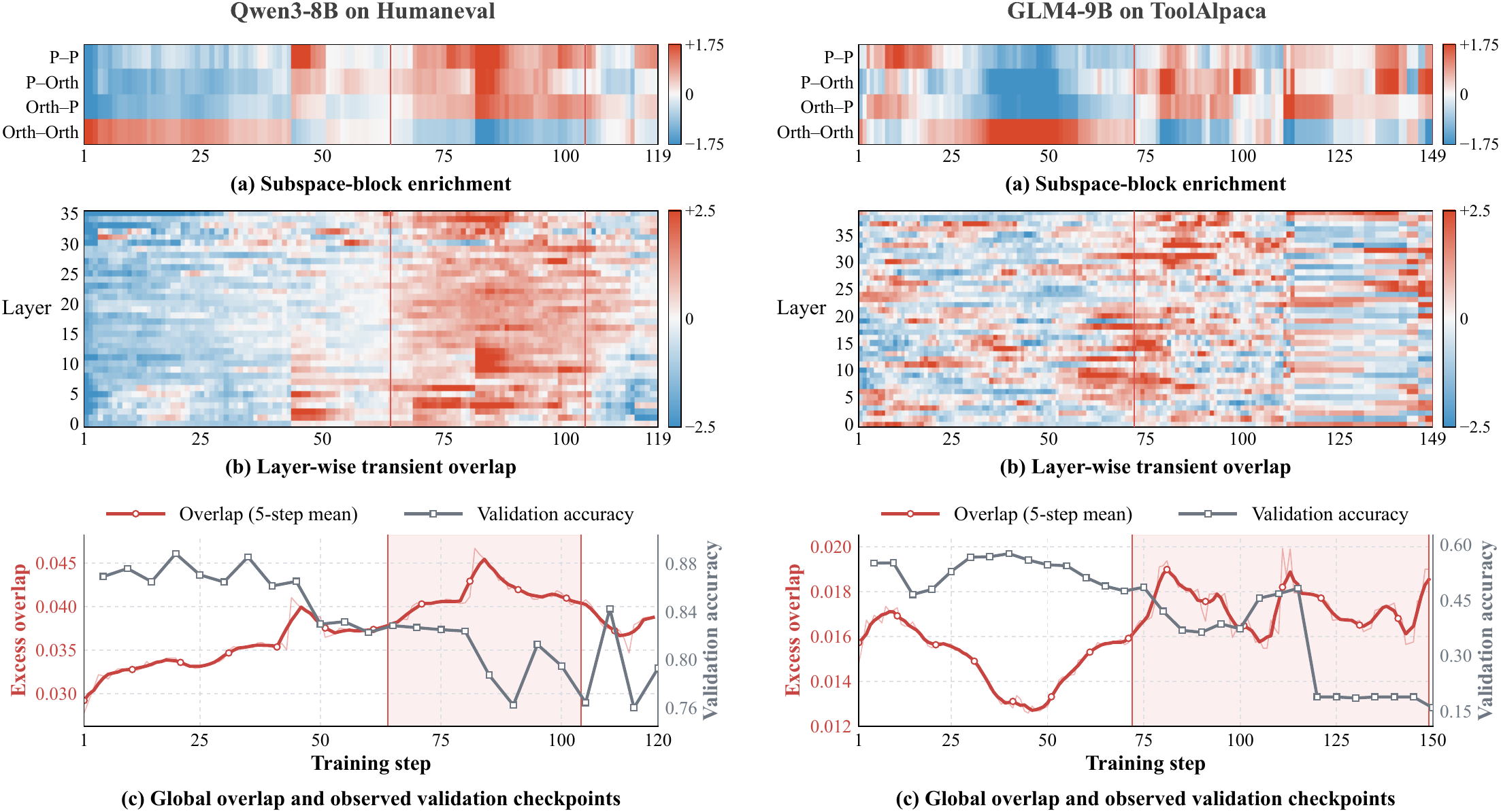}
    \caption{Stepwise update overlap and validation performance. The red curve is the 5-step moving average of excess principal-subspace overlap; the blue curve is validation accuracy. In both runs, episodes of elevated overlap accompany subsequent gradual degradation. This observation is correlational.
}
    \label{fig:additional_overlap}
\end{figure*}

\subsection{Additional Controlled Interventions}
\label{appendix:intervention}

We repeat the layer-wise norm-matched intervention on three
additional model--task--checkpoint configurations. As shown
in Figure~\ref{fig:additional_interventions_1} and~\ref{fig:additional_interventions_2}, increasing the
principal-overlapping component consistently produces a
dose-dependent accuracy drop across settings. These results
support the robustness of the local intervention effect beyond
the main configuration.

\begin{figure}[t]
    \centering
    \includegraphics[width=\linewidth]{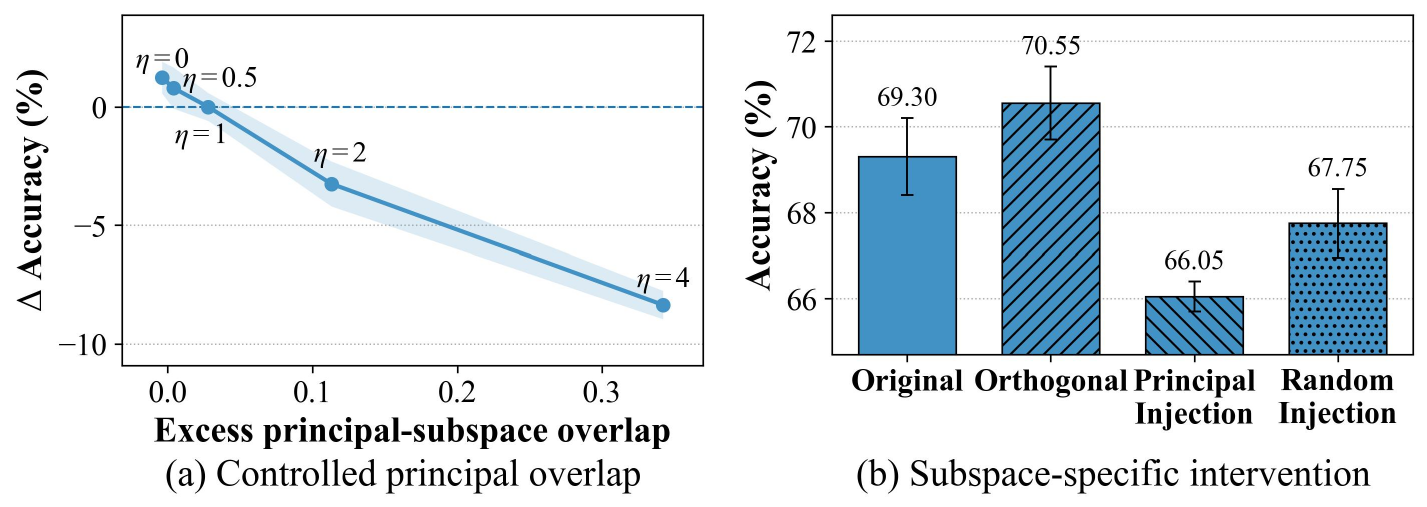}
    \caption{
    Additional controlled intervention on the step-125 GRPO update of Qwen3-8B on MATH500. (a) Increasing the principal-overlapping component under
    layer-wise norm matching produces a dose-dependent accuracy drop.
    (b) Orthogonalization improves accuracy, whereas matched
    principal-subspace injection is substantially more harmful than a
    random-subspace intervention.
    }
    \label{fig:additional_interventions_1}
\end{figure}

\begin{figure}[t]
    \centering
    \includegraphics[width=\linewidth]{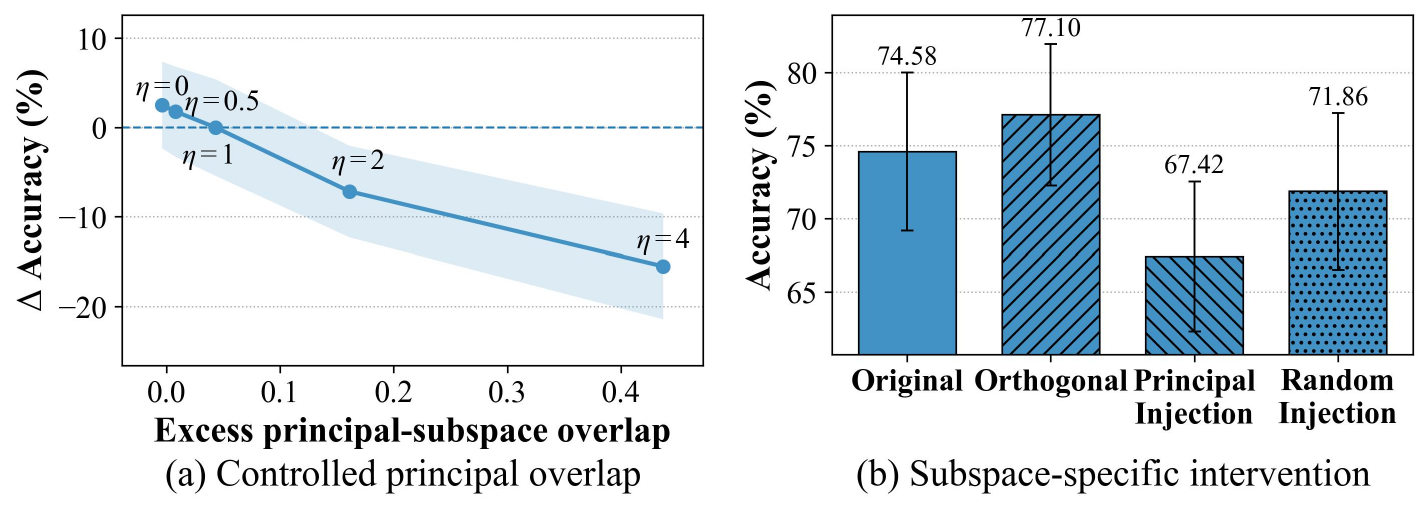}
    \caption{
    Additional controlled intervention on the step-70 GRPO update of GLM on Humaneval. (a) Increasing the principal-overlapping component under
    layer-wise norm matching produces a dose-dependent accuracy drop.
    (b) Orthogonalization improves accuracy, whereas matched
    principal-subspace injection is substantially more harmful than a
    random-subspace intervention.
    }
    \label{fig:additional_interventions_2}
\end{figure}

\end{document}